\documentclass{article}

\usepackage[english]{babel}

\usepackage{amsmath}
\usepackage{amsfonts}
\usepackage{amsthm}
\newtheorem{theorem}{Theorem}
\usepackage{graphicx}

\usepackage{tikz}
\usetikzlibrary{shapes, arrows.meta,arrows}
\tikzset{>=latex}

\usepackage{wrapfig}
\usepackage{caption}
\usepackage{subcaption}
\usepackage[colorlinks=true, allcolors=blue]{hyperref}
\usepackage{pgfplots}
\pgfplotsset{compat=1.17}
\usepgfplotslibrary{fillbetween}
\usepackage{dblfloatfix}  

\usepackage{style/fancyhdr}
\usepackage[accepted]{style/icml2024}
\usepackage{style/algorithmic}
\usepackage{style/algorithm}

\newcommand{\y}[1]{%
  \ifcase#1
    0
  \or
    -0.9
  \or
    -1.9
  \or
    -2.8
  \or
    -4.8
  \or
    -5.7
  \else
    -6.7
  \fi
}
\newcommand{\x}[1]{%
  \ifcase#1
    0
  \or
    1.5
  \or
    2.65
  \or
    3.5
  \or
    4.8
  \or
    6.4
  \else
    6
  \fi
}

\newcommand{\offset}{0.05}

\definecolor{green2}{RGB}{0, 163, 8}
\definecolor{lightblue}{RGB}{200, 255, 255}
\definecolor{lightgreen}{RGB}{200, 255, 200}
\definecolor{darkgreen}{RGB}{0, 130, 0}
\definecolor{lightred}{RGB}{255, 180, 180}
\definecolor{lightpurple}{RGB}{255, 180, 255}

\newcommand{\customplot}[6]{
    \addplot[#1] table [col sep=comma, x=step, y=#2] {#5};
    \addlegendentry{#6}
    \addplot[#1, opacity=0.0, name path=#2_lower, forget plot] table [col sep=comma, x=step, y=#3] {#5};
    \addplot[#1, opacity=0.0, name path=#2_upper, forget plot] table [col sep=comma, x=step, y=#4] {#5};
    \addplot[#1, opacity=0.3, forget plot] fill between[of=#2_lower and #2_upper];
}

\newcommand{\specialcomment}[1]{\hfill$\triangleright$\textit{#1}}

\renewcommand{\ss}[0]{\mathcal{S}}
\newcommand{\as}[0]{\mathcal{A}}

\begin{document}

\twocolumn[
\icmltitle{Zero-Shot Reinforcement Learning via Function Encoders}

\begin{icmlauthorlist}
    \icmlauthor{Tyler Ingebrand}{ut}
    \icmlauthor{Amy Zhang}{ut}
    \icmlauthor{Ufuk Topcu}{ut}
\end{icmlauthorlist}

\icmlaffiliation{ut}{University of Texas at Austin}
\icmlcorrespondingauthor{Tyler Ingebrand}{tyleringebrand@utexas.edu}
\icmlkeywords{Machine Learning, Reinforcement Learning, Representations, Linear}
\vskip 0.15in
]
\printAffiliationsAndNotice{}

\begin{abstract}

Although reinforcement learning (RL) can solve many challenging sequential decision making problems, achieving \textit{zero-shot} transfer across related tasks remains a challenge. 
The difficulty lies in finding a good representation for the current task so that the agent understands how it relates to previously seen tasks.
To achieve zero-shot transfer, we introduce the \textit{function encoder}, a representation learning algorithm which represents a function as a weighted combination of learned, non-linear basis functions. 
By using a function encoder to represent the reward function or the transition function, the agent has information on how the current task relates to previously seen tasks via a coherent vector representation. 
Thus, the agent is able to achieve transfer between related tasks at run time with no additional training. 
We demonstrate state-of-the-art data efficiency,
asymptotic performance, and training stability in three RL fields by augmenting basic RL algorithms with a function encoder task representation.
\end{abstract}

\section{Introduction}
\let\thefootnote\relax\footnotetext{Code: \scriptsize{\href{https://github.com/tyler-ingebrand/FunctionEncoderRL}{https://github.com/tyler-ingebrand/FunctionEncoderRL}}}

While deep reinforcement learning (RL) has demonstrated the ability to solve challenging sequential decision making problems, many real-life applications require the ability to solve a continuum of related tasks, where each task has a fixed objective and dynamics function.
For example, an autonomous robot operating in a kitchen needs the ability to achieve various cooking and cleaning objectives, each of which has a separate reward function. 
Likewise, if the robot is operating outside during winter, it must be able to operate in various slippery conditions, each of which has a separate transition function. 
However, it is not possible to learn a policy using standard RL algorithms for every possible task because there are conceivably infinite variations of reward and transition functions for a given system.

A key desiderata for learning systems is \textit{zero-shot} transfer, the ability to solve any problem from the task continuum at run time with no additional training. 
Zero-shot transfer would allow the robot to solve all of its kitchen objectives without retraining by reusing information from similar tasks. 
Likewise, the robot would be able to walk on a slippery surface by slightly modifying policies capable of walking on similar surfaces.
In order for an autonomous robot to achieve zero-shot transfer, it must know which reward function it should be optimizing and the properties of its current transition function. 
In other words, the autonomous system needs an informative task description that uniquely identifies the current task and describes how an unseen task relates to prior tasks. 

Prior works in zero-shot RL identify the task through a context variable, which is either   given \cite{her} or calculated from data \cite{fb, contextualRL, context2}. 
Such context variables are often domain specific \cite{usfa, hipmdplatent} and lack out-of-distribution guarantees with respect to related but unseen tasks. 
In contrast, we seek an algorithm that is applicable to many domains, and a context representation that can provably generalize to unseen but related contexts.  
Recent works have also described the task via natural language, and trained a policy through imitation learning \cite{rt, rt2}. However, this approach requires an enormous amount of data which is impractical for most use cases.

\begin{figure}           
    \centering
    \resizebox{\linewidth} {!} {%
    \includegraphics[]{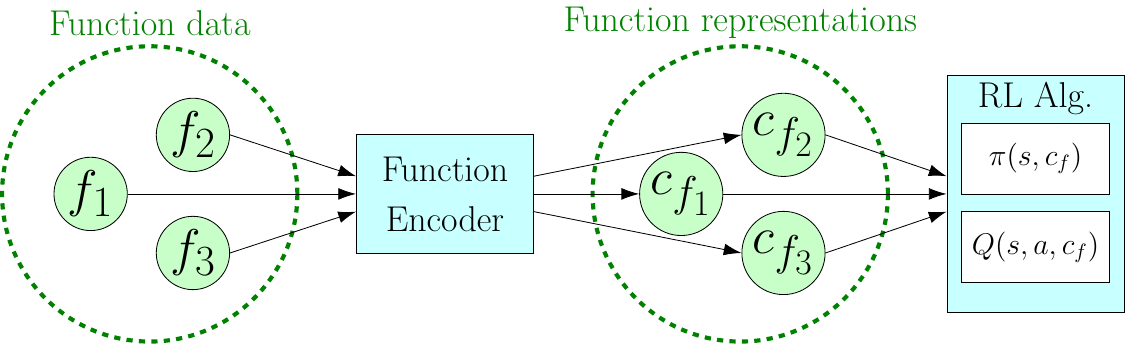}
    }

    \caption{A diagram representing the workflow of function encoders. 
    The set of functions is converted into a set of representations via a function encoder. 
    Those representations are passed into the RL algorithm as input to the policy and value functions. The represented functions can be reward functions and/or transition functions, depending on the setting.}
    \label{workflow}
    \vspace{-10pt}

\end{figure}

In this paper, we introduce the \textit{function encoder}, a representation learning algorithm which can be seamlessly combined with any RL algorithm to achieve zero-shot transfer in sequential decision making domains. 
Our algorithm first learns a set of non-linear basis functions which span the space of tasks, where a given task is represented by a function. 
New tasks are described as a linear combination of these basis functions, thus identifying how the current task relates to previously seen tasks. 
Once we have found the basis function coefficients for a new task, we pass those coefficients as a context variable into the policy.
This allows the policy to transfer to new tasks because similar tasks often have similar optimal policies, and the function encoder represents similar tasks with similar coefficients by design.
Thus, a basic RL algorithm which is augmented with a function encoder task representation as an additional input is able to adapt its policy to the given task. 

To demonstrate the broad applicability of our approach, we perform a diverse set of experiments in hidden-parameter system identification, multi-agent RL, and multi-task RL.
In a challenging hidden-parameter version of the Half-Cheetah environment, our approach shows a 37.5\% decrease in the mean square error of dynamics prediction relative to a transformer baseline. 
In a multi-agent tag environment, our approach shows significantly better asymptotic performance than comparable baselines, while matching the performance of a one-hot encoding oracle. 
In a multi-task version of Ms. Pacman, our approach shows a 20\% higher success rate relative to multi-task algorithms and better data efficiency than a transformer baseline. 

Additional qualitative analysis shows the similarity between learned task representations directly reflects the similarity between the tasks themselves.  
For example, Half-Cheetah environments with similar hidden variables will have similar representations (as measured by cosine similarity), while environments with large differences will have dissimilar representations. 
Our representation learning algorithm successfully encapsulates the relationships between tasks, allowing basic RL algorithms to achieve zero-shot transfer. 

\paragraph{Contributions}
\begin{itemize}
    \item We introduce a novel, general-purpose representation learning algorithm which finds representations for every function in a space of functions.
    \item We show that the algorithm achieves state-of-the-art performance in a supervised learning setting despite being computationally simple. 
    \item We demonstrate that the learned representations are \textit{widely applicable} and can be combined with any RL algorithm for zero-shot RL.
\end{itemize}

\section{Related Works}
\paragraph{Zero-shot RL}
There are three typical approaches to zero-shot RL, where each episode is modeled as a related but unique Markov decision process (MDP). We define context as the information needed to adapt a policy to the current episode's underlying MDP. In some works, the context is known \cite{her}. We consider the case where the context is  unknown but it may be implicitly described by data. 

The first approach is to find a policy which maximizes the worst-case performance under any context. This approach is required when there is no data to identify the current episode's context. Robust RL \cite{robust} and most multi-agent RL algorithms \cite{starcraft, selfplay} follow this approach. Many of these algorithms train a RL algorithm on numerous contexts simultaneously, such as an agent playing against an adversary randomly drawn from a league of adversaries. 

The second approach is to compute a context representation from data, and adapt the policy via the context representation. Many works in multi-task RL and hidden-parameter RL take this approach \cite{ hipmdps, contextualRL, fb, usfa, sf, pearl, hipmdp_latent}. Prior works lack guarantees about how representations will transfer to related but unseen tasks. In contrast, our approach guarantees a good representation for unseen tasks so long as they are a linear combination of the learned basis functions. 

The third approach is to directly include data on the current episode or task as input to the policy, often through a sophisticated architecture like a transformer \cite{transformerMetaRL, rl2, dt, rt}. 
This strategy is motivated by the fact that transformers have proven to be effective in natural language processing \cite{gpt, bert} and in sequential decision making problems \cite{rt, dt}, where a large amount of data is processed simultaneously. 
However, transformers have increased costs with respect to memory usage, training time, data efficiency, and training stability compared to other approaches \cite{traintrans, traintrans2}.
All transformer baselines in our experiments fall into this category. 
 
\paragraph{Basis Functions}
Prior works, such as the Fourier series or Taylor series, describe basis functions which can approximate functions with arbitrary precision.
The corresponding coefficients can in principle be used as representations for functions. 
However, these analytical approaches suffer from the curse of dimensionality and perform poorly on high-dimensional function spaces.
Additionally, high-dimensional function spaces, such as images or sensor data, are theorized to occupy low-dimensional manifolds \cite{manifold}.
Learned basis functions can fit only this manifold without representing every possible function in that space. 
In other words, learned basis functions may better fit the data with a relatively small number of basis functions compared to analytical approaches.

There are prior works from transfer learning which investigate learned basis functions. One approach specifies basis functions in the same form as a function encoder but computes the coefficients via a deep neural network \cite{learnedbasis}. In contrast, the function encoder computes the coefficients through the inner product, which is computationally efficient and ensures the function encoder is a linear operator. Another work learns basis functions as a space of features for classification problems \cite{protonetwork}. Our work involves similar basis function design, but is applicable to regression problems. 
This is motivated by the continuous nature of our setting where classification algorithms are ill-suited.

\section{Preliminaries}
We denote the reals as $\mathbb{R}$ and expectation as $\mathbb{E}\left[\cdot\right]$. 
Calligraphic characters such as $\mathcal{R}$ indicate sets. 

A Markov decision process (MDP) $m$ is a tuple $(\ss, \as, T, R)$ where $\ss$ is the state space, $\as$ is the action space, $T: \ss \times \as \mapsto \ss$ is the transition function and $R: \ss \times \as \mapsto \mathbb{R}$ is the reward function. 
The initial state at time zero is sampled from a set of initial states $\ss_0 \subseteq \ss$, and the next state is determined by the transition function. 
The agent receives reward according to the reward function \cite{suttonbarto}. 

The objective for the agent is to find the optimal policy $\pi^*: \ss \mapsto \as$ which maximizes $\mathbb{E} \left[ \sum_{t=0}^{\infty} \gamma^t R(s_t, a_t) \right]$, the expectation of accumulated discounted reward for future states according to some discount factor $\gamma$.  
The function $V^\pi(s_t) = \mathbb{E} \left[ \sum_{n=t}^{\infty} \gamma^n R(s_n, a_n) \right]$ with $a_t \sim \pi(s_t)$ is the state-value function for policy $\pi$, and the function $Q^\pi(s_t,a_t) = \mathbb{E} \left[ R(s_t,a_t) + \gamma V^\pi(s_{t+1})\right]$ is the state-action-value function \cite{suttonbarto}. 
There is a rich literature for how to find the optimal policy via RL \cite{dqn, ddpg, d4pg, ppo, sac}. 
In order to take advantage of these prior works, we will describe an algorithm which is generally applicable.

\section{The Function Encoder}
\subsection{Motivation}
Many fields of RL solve a modified MDP where each episode varies with respect to a function.
We define a function which varies every episode and affects the optimal policy as a \textit{perturbing function}.
The perturbing function view of RL is widely applicable. In multi-task RL, the reward function $r$ is sampled from a set of reward functions, and a change in reward function causes a change in the optimal policy. 
Thus, the reward function is a perturbing function in multi-task RL. 
In hidden-parameter RL, the transition function varies every episode due to the hidden parameters. 
A change in transition function affects the optimal policy and therefore the transition function is a perturbing function. 

In general, there is no closed form solution for either the optimal policy or the value function in terms of a perturbing function. 
A small change in the perturbing function can sometimes lead to abrupt and discontinuous changes in the optimal policy. 
However, it is often the case that a change in the perturbing function leads to a small, continuous change in the optimal policy. 
In other words, the relationship  between perturbing functions and optimal policies tends to be piece-wise continuous with sparse discontinuities.
See \ref{RewTranPol} for an example of how this arises even in simple settings. 
It is possible to contrive an example where this is not the case, but we do not observe this in practice. 

Since the perturbing function affects the optimal policy, the RL algorithm must have rich information on this function to calculate the optimal policy.
Therefore, we give the RL algorithm information on the perturbing function via a learned representation which is sufficient to distinguish between perturbing functions.
Thus, we represent every perturbing function in a space of  perturbing functions, where we are given some data to calculate this representation.  
Section \ref{sec:fe} describes how to find a representation for every  function in a space of perturbing functions. Section \ref{sec:zeroshot} describes how to use this representation for zero-shot RL. 

\subsection{Training a Function Encoder} \label{sec:fe}

Consider a set of functions $\mathcal{F} = \{f | f: \mathcal{X} \mapsto \mathbb{R}\}$ where the input space $\mathcal{X} \subset \mathbb{R}^n$ has finite volume.
This function set represents the set of perturbing functions.
Note that when there are $m>1$ output dimensions, we apply the same procedures $m$ times independently. 
Suppose $\mathcal{F}$ is a Hilbert space with the inner product $\langle f, g \rangle = \int_\mathcal{X} f(x)g(x)dx$, 
then there exists a set of $k$ orthonormal basis functions $\{g_1, g_2, ..., g_k\}$ such that for any $f \in \mathcal{F}$, 
\begin{equation}
    f(x) = \sum_{i=1}^k c_i g_i(x),\label{eq:1}
\end{equation}
where $c_i$ is the coefficient for basis function $i$ and $k$ may be infinite. \cite{functionalanalysis}. 
Given $f$, there is only a single sequence of coefficients that satisfy the equation due to the orthonormality of the basis functions. 
Given the coefficients, one can recover $f(x)$ via \eqref{eq:1}. 
Therefore, the coefficients are a unique representation of the function. 

Given the basis functions $\{g_1, g_2, ..., g_k\}$, we compute the coefficients as 
\begin{equation}
c_i = \langle f, g_i \rangle = \int_{\mathcal{X}} f(x) g_i(x) dx, \label{eq:2} 
\end{equation}

which comes from the definition of the inner product. See \ref{proof.eq2} for a derivation.
For high-dimensional $f$, this integral is intractable. 
However, given a dataset $D = \{(x_j, f(x_j))| j=1,2,...\}$, the coefficients are approximated using Monte-Carlo integration as 

\begin{equation}
c_i \approx \frac{V}{|D|} \sum_{x, f(x) \in D} f(x)g_i(x), \label{eq:3}
\end{equation}

where $V$ is the volume of the input space. 
Monte Carlo integration requires the input data points $\{x_j | j = 1,2,...\}$ to be uniformly distributed throughout the input space. If the data set is not uniformly distributed, it can be corrected with importance sampling \cite{montecarlo}. Additionally, as $|D|$ approaches infinity, the error in the approximation approaches zero \cite{montecarlo}.
Since the coefficients uniquely identify the function,  we find a unique representation for a function $f$ given basis functions $\{g_1, g_2, ..., g_k\}$ and data on the function $f$. 

\begin{algorithm}[b]
  \caption{Function Encoder}
  \label{alg:fe}
  \begin{algorithmic}[1]

    \STATE \textbf{Input:} Step size $\alpha$, set of data sets $D = \{\{(x_i, f_j(x_i)|i=1,2,..., I\}|j=1,2,..., J\}$
    \STATE \textbf{Output:} Basis functions $\{\hat{g}_1, \hat{g}_2, ..., \hat{g}_b\}$
    \STATE Initialize $\{\hat{g}_1, \hat{g}_2, ..., \hat{g}_b\}$ parameterized by $\theta$

    \WHILE{not converged}
        \STATE $loss = 0$
        \FOR{$\{(x_i, f_j(x_i)|i=1,2,...\}$ in $D$}
            \STATE $(\hat{c}_{f_j})_k = \frac{V}{I}\sum_{x_i,f_j(x_i)} f_j(x_i) \hat{g}_k(x_i) \;\; \forall k $ \specialcomment{Eq. 4}
            \STATE $\hat{f}_j(x_i) = \sum_{k=1}^b (\hat{c}_{f_j})_k \hat{g}_k(x_i)\;\;\; \forall i$ \specialcomment{Eq. 5}
            \STATE $loss \mathrel{+}= (\hat{f}_j(x_i) - f_j(x_i))^2 / I \;\;\; \forall i$ \specialcomment{Eq. 6\:}
        \ENDFOR
        \STATE $\theta = \theta - \alpha \nabla_\theta loss$
        
    \ENDWHILE
    \STATE \textbf{return} $\{\hat{g}_1, \hat{g}_2, ..., \hat{g}_b\}$
  \end{algorithmic}
\end{algorithm}

\begin{figure}[b!]
    \centering
    \resizebox{5.5cm} {!} {%
        \includegraphics{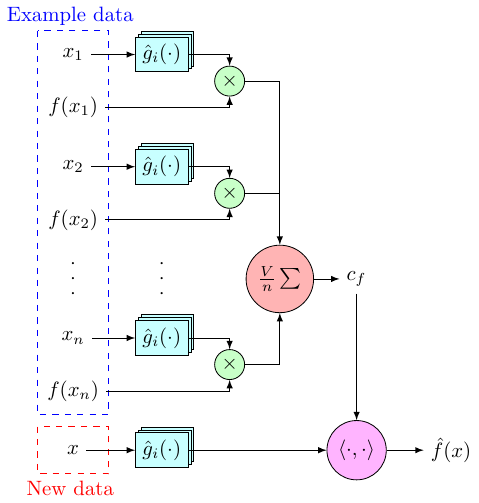}
        }
    \vspace{-10pt}
    \caption{A block diagram representing the flow of information in a function encoder. The top segment of the diagram shows how to use example data to compute the representation $c_f$. The bottom segment shows how to use $c_f$ to predict $\hat{f}(x)$ for a given input $x$. 
    }
    \label{diagram}
\end{figure}

This approximation is an important aspect of the function encoder. 
The representation for a function is calculated using a sample mean (scaled by $V$), which can be computed efficiently on a GPU. 
A large batch of data can be used to compute a representation in only milliseconds.
It is also possible to iteratively update this sample mean as new data arrives, such that a low-compute embedded system could calculate this representation in real-time with constant memory usage. 
Lastly, once we have computed the coefficients from data, we can compute $f$ via \eqref{eq:1} without any form of retraining. 
Thus, the function encoder is extremely useful for online settings.

The only remaining challenge is how to find the basis
functions for a space of unknown functions.
To do so, we first initialize $b$ basis function approximations $\{\hat{g}_1, \hat{g}_2, ..., \hat{g}_b\}$ using neural networks (or one multi-headed neural network). 
Initially, these basis functions neither span the function set nor capture any relevant information. 
Nonetheless, we compute the coefficients using a dataset $D$,
\begin{equation}
    \hat{c}_i = \frac{V}{|D|}\sum_{x,f(x) \in D} f(x) \hat{g}_i(x). \label{eq:4}
\end{equation}

Once we have computed the coefficients, we approximate $f$ using the basis function approximations,
\begin{equation}
    \hat{f}(x) = \sum_{i=1}^b \hat{c}_i \hat{g}_i(x). \label{eq:5}
\end{equation}

Lastly, we define a loss function for the function approximation, such as the mean squared error
\begin{equation}
   L(D) = \frac{1}{|D|}\sum_{x,f(x) \in D} (\hat{f}(x) - f(x))^2, \label{eq:6}  
\end{equation}

which is minimized via gradient descent. 
Following this process in a iterative fashion yields Algorithm \ref{alg:fe}. 
The result is a set of learned, non-linear basis functions which span the set of functions. 
We call the set of learned basis functions a \textit{function encoder}, since the basis functions encode any function $f \in \mathcal{F}$ into a vector representation $c_f = \{c_1, c_2, ..., c_b\}$. See Figure \ref{diagram} for a graphical representation of how example data is used to predict $\hat{f}(x)$.

\begin{theorem}The function encoder's mapping from functions to representations is a linear operator.
\end{theorem}
\begin{proof}
Consider a function $f_3 = a f_1 + b f_2$ where $a \in \mathbb{R}$ and $b \in \mathbb{R}$. 
The $i$-th coefficient for function $f_3$ can be computed using \eqref{eq:2}: 
\[ (c_{f_3})_i = \langle f_3, g_i \rangle \]
\[ (c_{f_3})_i =\langle a f_1+ b f_2,  g_i \rangle \]
\[ (c_{f_3})_i = a  \langle f_1, g_i \rangle  + b \langle f_2, g_i \rangle \]
\[ (c_{f_3})_i = a (c_{f_1})_i + b (c_{f_2})_i\]
This is true for every basis function $g_i$, so therefore it is true for the vector representation. 
Thus, the linear relationship between functions is preserved as a linear relationship between representations, $c_{f_3} = a c_{f_1} + b c_{f_2}$.
\end{proof}

This implies if $f_3$ is not a function in the training dataset, but $f_1$ and $f_2$ are, then $f_3$ can be well represented. 
If the function encoder can represent every function in the training set, then it can also represent unseen functions so long as they are a linear combination of functions in the training set. 
Furthermore, this implies it possible to increase the dimensionality of the learned space by incorporating diverse training functions.
Thus, the function encoder yields unique representations with predictable and generalizable relationships. 

\subsection{Orthonormality}
This algorithm does not enforce orthonormality. Empirically, we observe that the basis functions converge towards orthonormality,  where an orthonormal basis spanning the function space has zero loss. See \ref{app:ortho} for a discussion.

\subsection{Zero-Shot RL via Function Encoders} \label{sec:zeroshot}

To achieve zero-shot transfer in a RL domain, we first encode the perturbing function using a function encoder. 
The representation uniquely identifies the perturbing function and its relationship to previously seen functions.
The representation is passed into a RL algorithm as an additional input, which yields a policy of the form $\pi(s, c_f)$ and value functions of the form $V^\pi(s,c_f)$ and $Q^\pi(s,a,c_f)$ where $c_f$ is the encoding of the perturbing function. 
Because policies and value functions are common components in all RL algorithms, this approach is widely applicable. 
Providing the representation allows the RL algorithm to successfully adapt its actions depending on the current episode's perturbing function, as we demonstrate in Section \ref{sec:experiments}.

\paragraph{A Key Assumption} Data on the perturbing function is needed to compute its representation. 
This data has the form of input-output pairs, but no further information is needed on the perturbing function neither during training nor execution. 
This also implies some exploration must be done each episode, to collect data, before exploitation can occur. 
This paper does not address the exploration problem and assumes access to  data on the perturbing function.

\begin{figure}[!b]
        \centering
        \includegraphics[width=1.0\linewidth]{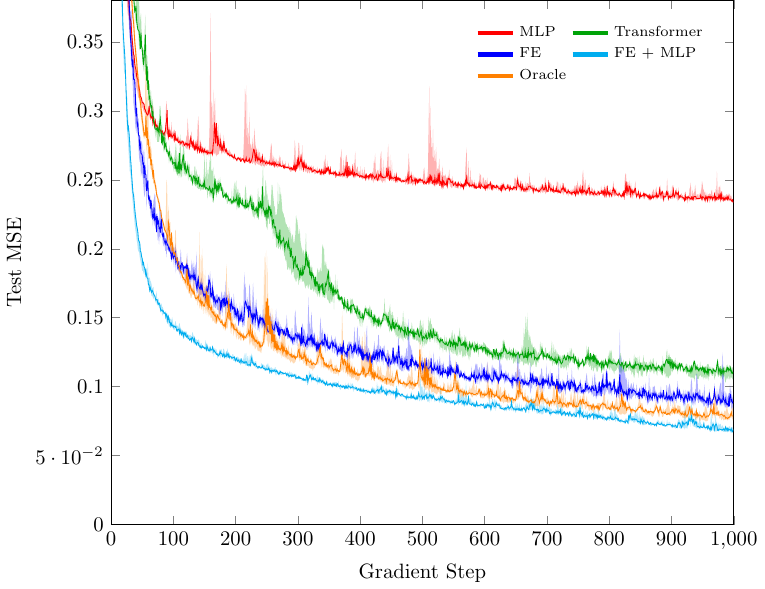}
        \vspace{-20pt}
        \caption{Comparison of MLPs, transformers, and function encoders on system identification of a hidden-parameter MDP. Each algorithm is run for three seeds, with the shaded areas representing minimum and maximum values. 
        }
        \label{sys_id_loss}
\end{figure}%

\section{Experiments} \label{sec:experiments}
To evaluate our approach, we first ensure that a function encoder can be accurately trained in a supervised setting. In Section \ref{sec:hipsysid}, we demonstrate faster convergence and better asymptotic performance, relative to a transformer baseline, on a supervised hidden-parameter system identification problem. Next, we evaluate the quality of the representations created by a function encoder. In Sections \ref{sec:ma} and \ref{sec:mt}, we demonstrate zero-shot RL by passing the representation of the perturbing function into the RL algorithm. In order for a policy to perform well in these settings, it requires rich information on the perturbing function, and thus the results indicate that the representations carry this rich information. See Appendix \ref{implementationdetails} for implementation details.
We use $b=100$ basis functions for all experiments.
See Appendix \ref{sec:ablations} for an ablation on how the hyper-parameters affect performance.

\begin{figure*}[!h]
    \centering
    \includegraphics[width=1.8\columnwidth]{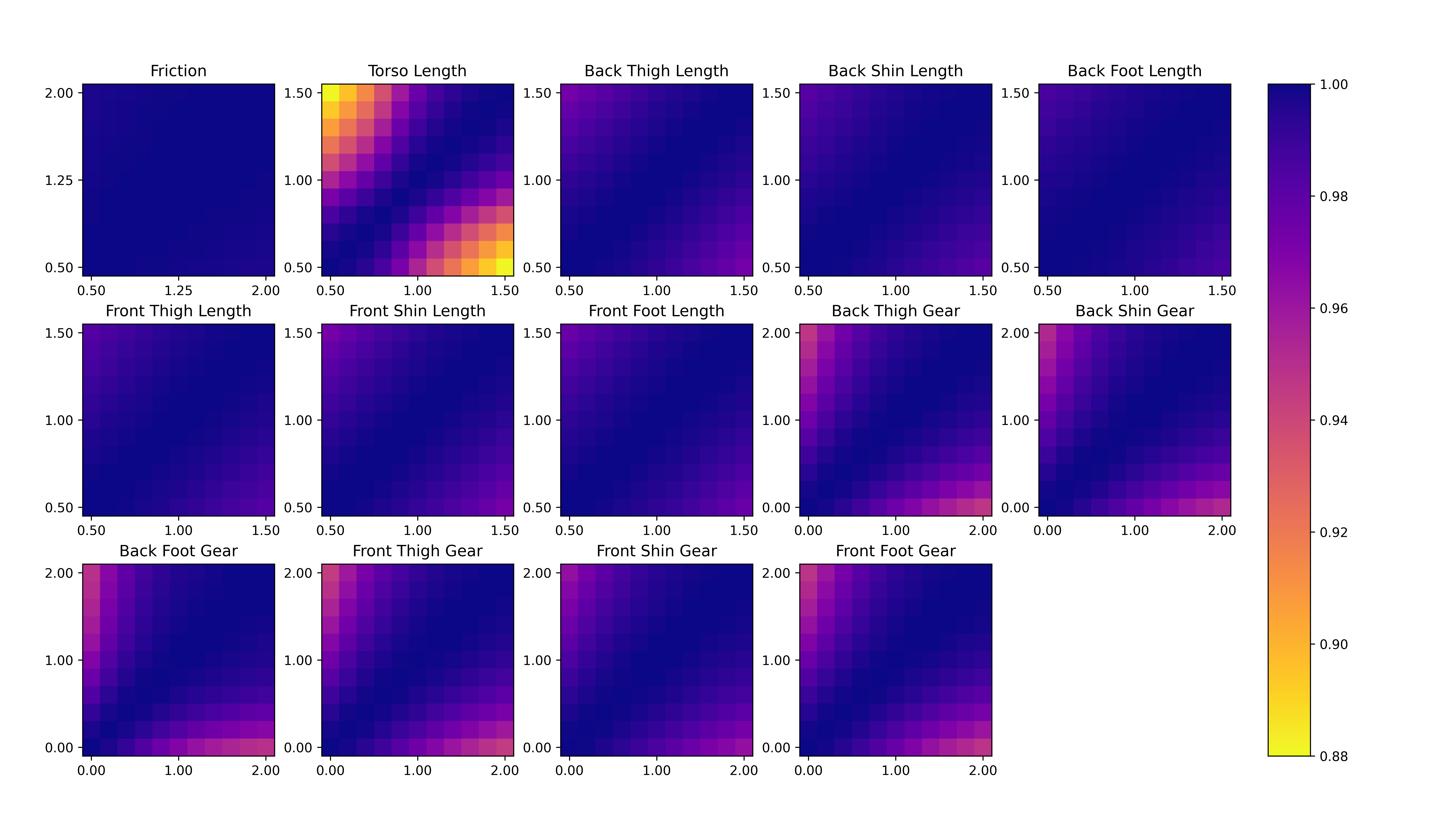}
    \vspace{-28 pt}
    \caption{A plot of cosine similarity between function encoder representations for hidden-parameter environments. 
    Axes show the hidden parameter value as a ratio of its default value in the Half-Cheetah environment. 
    This figure shows that the function encoder representations directly relate to the underlying hidden parameters in a consistent fashion, where an increasing change in a given hidden parameter leads to an increasing change in the representation. 
    }
    \label{sys_id_cos_sim}
    \vspace{-10pt}
\end{figure*}

\subsection{Hidden-Parameter System Identification} \label{sec:hipsysid}

Hidden-parameter MDPs differ from MDPs in that the transition function depends on an additional hidden parameter $\theta$. The hidden parameter $\theta$ varies every episode and is unknown to the agent. From the agent's perspective, the transition function is different every episode and this affects the optimal policy. 
Thus, the transition function is a perturbing function in this setting. 

We compare function encoders against two other deep learning baselines for system identification in hidden-parameter MDPs. 
The testing environment is a modified Half-Cheetah environment \cite{towers_gymnasium_2023} where the segment lengths, friction coefficient, and control authority are randomized within a range each episode, which leads to variance in the transition function. 
The goal is for the system identification algorithm to accurately predict transitions given $5,000$ example data points on previous transitions. The training dataset includes $200$ transition functions. Figure \ref{sys_id_loss} plots the results. 

\texttt{MLP} cannot incorporate example data, so its lowest MSE estimator would be to predict the average transition function in its data set. Its performance stalls because it is not possible to accurately predict the transition function without using information on the hidden parameters. 

\texttt{Transformer} can incorporate the example data by passing it as input into the encoder side of the transformer.
Unlike the function encoder, the transformer is memory inefficient so it is not able to use all of the example data. 
The state-action pair, for which we want to predict the next state, is input to the decoder side of the transformer. Note that transformers are computationally expensive, and suffer a $194\%$ increase in training time relative to \texttt{MLP}.

\texttt{FE} is able to use all example data by converting it into a function encoder representation. This representation can then be used to estimate the function, as shown in \eqref{eq:5}. We observe that the function encoder shows better performance relative to the two baselines, with a 19.7\% decrease in MSE relative to the transformer. Additionally, the function encoder is computationally efficient and only incurs a $5\%$ increase in training time relative to \texttt{MLP}. 

\texttt{FE + MLP} is an extension where the function is represented as $f(x) = \Bar{f}(x) + f_{dif}(x)$, where $\Bar{f}$ is the average transition function and $f_{dif}$ is the difference between the current function and the average function. $\Bar{f}(x)$ is a MLP trained via a typical gradient-based approach, whereas $f_{dif}(x)$ is a function encoder. This approach has better data efficiency than a standalone function encoder because the data is only needed to predict how the current function differs from the average function, which is an easier task than identifying the function itself. \texttt{FE + MLP} achieves a 37.5\% decrease in MSE relative to the transformer baseline. However, there is a gradient calculation for both $\Bar{f}(x)$  and $f_{dif}(x)$, which leads to a moderate $75\%$ increase in training time relative to \texttt{MLP}.
We would like to highlight that this approach achieves good performance with as little as $50$ data points. See \ref{sec:ablations}.

\texttt{Oracle} is a MLP baseline with access to the hidden parameters as a input variable. 
The oracle is an \textit{approximate} upper bound on the performance of an end-to-end system identification algorithm because it is provided with all of the information needed to accurately predict the dynamics. 

Function encoders also allow us to compare the representations across environments with different hidden parameters. 
An ideal representation algorithm would show a high cosine similarity between two environments with similar hidden parameters, and a low cosine similarity between two environments with divergent hidden parameters.
Figure \ref{sys_id_cos_sim} shows the cosine similarity between environments that vary along a single hidden parameter dimension for the Half-Cheetah experiment. 
We observe the desired relationship for the function encoder's representation, and we can additionally use the representation to study the environment. 
By analyzing which hidden parameters have the most effect on the representation, we can learn which hidden parameters have the most effect on the transition function itself since the representation directly corresponds to the transition function. 
The learned representation suggests  that torso length is the most influential factor on system dynamics, followed by control authority (gears).

\subsection[Multi-Agent Reinforcement Learning]{Multi-Agent Reinforcement Learning } \label{sec:ma}

Multi-agent RL models an environment where an adversary takes actions which affect the transitions and rewards. We assume the adversary's policy changes every episode, but remains fixed for a given episode. This assumption reflects an agent playing against a random opponent every episode. The adversary's policy should affect the agent's policy. For example, an adversary may take actions which can be exploited, and thus the optimal policy of the agent changes to exploit those weaknesses. Therefore, the adversary's policy is a perturbing function. 
At execution time, we assume access to $5{,}000$ data points on each adversary. 
The training dataset includes data on $10$ adversaries. 

We compare function encoders against three baselines for multi-agent RL in a partially observable game of tag \cite{terry2021pettingzoo}. One agent tries to maximize the distance between the two agents, while the other tries to minimize it. Furthermore, the agents' locations are not visible to each other, and so the tagger must guess where the runner is hiding. The ego agent plays against a random adversary each episode and the goal is to perform well against every adversary. 
We plot the results in Figure \ref{multi_agent_data}. 

\texttt{PPO} does not have access to any information on which agent it is playing against in the current episode. 
\texttt{PPO + OHE} gets access to a one-hot encoding of the index of its current adversary. Thus, it has access to information about which adversary it is playing against, but it cannot generalize these representations to new agents. 
\texttt{PPO + FE} uses adversary data to generate a representation of the adversary's policy, which is passed into the state-value function and the agent's policy. 
Furthermore, the function encoder can represent an unseen adversary via basis function coefficients, and so it could generalize to new adversaries with sufficient training.
Lastly, \texttt{PPO + Trans} uses this same adversary data as input to the encoder side of a transformer, while the environment state is passed to the decoder side of the transformer. 

Due to the partial observability of this environment, leveraging prior knowledge about a specific adversary is necessary to achieve the best possible performance. For example, if it is known that an adversary always moves to the same location, then this information can be exploited by the tagger. Both \texttt{PPO + FE} and \texttt{PPO + OHE} are capable of doing so as their policies have sufficient information to distinguish between adversaries, and consequently these two methods achieve the best performance. In contrast, \texttt{PPO} cannot distinguish between adversaries, and so its policy is suboptimal. \texttt{PPO + Trans} can theoretically distinguish between the adversaries by interpreting the adversary data. However, it converges much slower, suggesting that learning an optimal policy that incorporates adversary data directly is more challenging than learning from a principled representation.

\begin{figure}
        \centering
        \includegraphics[width=1.0\linewidth]{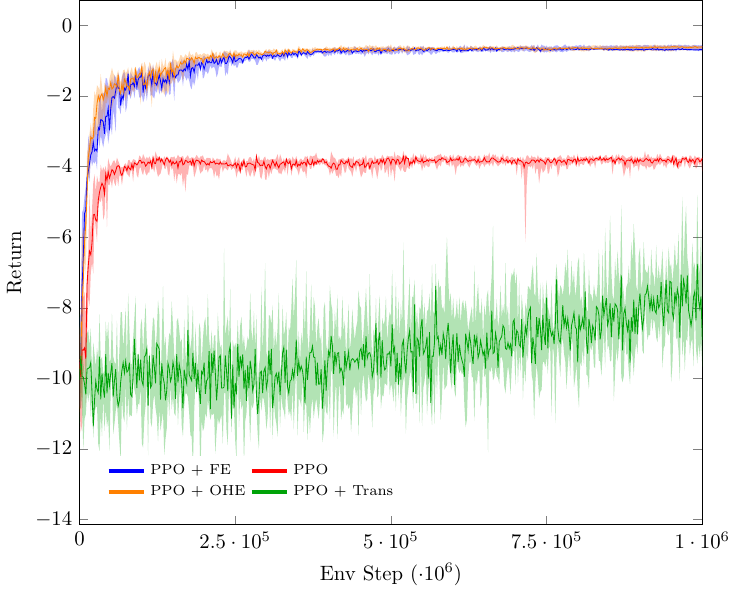}
        \vspace{-20pt}
        \caption{Training curves for four algorithms on a partially observable game of tag. The adversary is randomly sampled from a pre-trained league. Each algorithm is run for five seeds, with shaded areas indicating minimum and maximum values. }
        \label{multi_agent_data}
\end{figure}%

\begin{figure*}[!t]
    \centering
    \begin{subfigure}{0.43\textwidth}
    \includegraphics{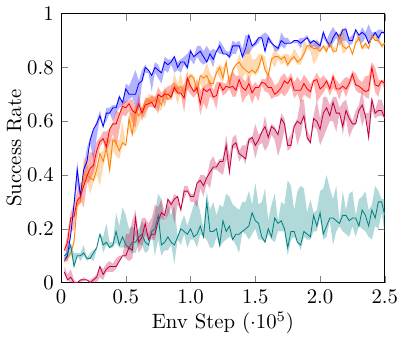}
    \end{subfigure}
    \hspace{-20pt}
    \begin{subfigure}{0.53\textwidth}
        \includegraphics{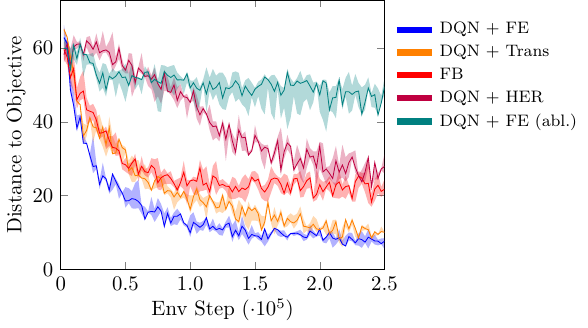}
    \end{subfigure}
    \vspace{-15pt}
    \caption{Comparison of forward-backward (FB) learning and various versions of DQN on the  Ms. Pacman environment. Left shows the fraction of episodes that terminate with Ms. Pacman at the goal location. Right shows the average distance to the goal location at the end of the episode. Shaded areas indicate the first and third quartiles over five seeds. }
    \label{multi_task_data}
\end{figure*}
\begin{figure*}[t]
    \centering
    \includegraphics[width=1.8\columnwidth]{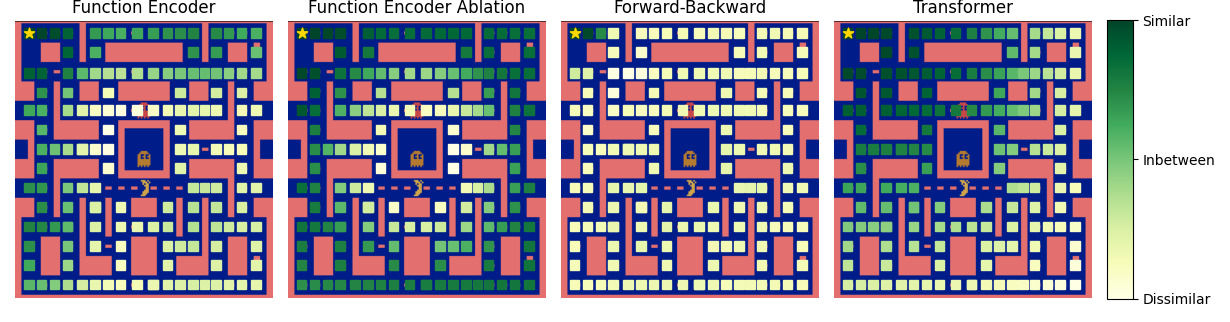}
    \vspace{-15pt}
    \caption{Cosine similarity between representations for reward functions of different goal locations. Similarity is shown between every goal location relative to the goal in the top left corner marked with a star. Cosine similarity scores are normalized for each algorithm.
    This figure shows that function encoders and transformers learn representations which maintain the relationships between goal locations.
    }
    \label{cos_sim_pacman}
\end{figure*}

\subsection{Multi-Task Reinforcement Learning} \label{sec:mt}

In multi-task RL, the reward function is sampled from a set of reward functions each episode. 
The sampled reward function affects the optimal policy. 
Therefore, the reward function is a perturbing function. 
In this section, we show that a function encoder can use data on the reward function to generate a representation. 
Then, this representation is passed into a RL algorithm to achieve zero-shot RL. 

We evaluate function encoders on a challenging multi-task version of Ms. Pacman \cite{fb}. 
The objective in this environment is to reach a goal location without being captured by a ghost.
However, the goal location is not given directly and instead the algorithm must infer the goal location from reward data. At execution time, we provide $10,000$ data points on the reward function to the algorithms.  The training set includes all reward functions. 

We compare against multi-task baselines, shown in Figure \ref{multi_task_data}. Forward-backward representation learning (\texttt{FB}) learns a representation of possible trajectories, and uses reward data to compute what the optimal trajectory should be \cite{fb}.
Note that \texttt{FB} uses no reward data during training, but uses the same reward information at execution time.
\texttt{DQN + Trans} uses the reward data as input to the encoder side and the current state as input to the decoder side of a transformer. 
\texttt{DQN + HER} is a multi-task algorithm which assumes access to the goal locations and the reward function \cite{her}.
\texttt{DQN + FE} uses reward data to compute a representation, which is then passed into the state-action-value function. 
\texttt{DQN + FE(abl.)} is an ablation where the representation $c_f$ is calculated according to \eqref{eq:4}, but instead of using \eqref{eq:5}, the reward function approximation is calculated as $\hat{r}(s,a) = \hat{r}_\theta(s,a, c_f)$, where $r_\theta$ is a MLP.
This representation is empirically not sufficient to accurately predict $r(s,a)$, and thus acts as a baseline with imperfect function representations. 

We observe that \texttt{DQN + FE} has better data efficiency and asymptotic performance compared to the other approaches. 
\texttt{DQN + FE} uses reward data during training to guide exploration, unlike \texttt{FB} which samples a reward function from its internal representation to guide exploration. 
Directly incorporating reward data via a transformer achieves good performance as well, although its worth noting that the transformer takes much more time to train (4x in this implementation) and requires hyper-parameter tuning. 
Lastly, the ablation shows that the quality of the representation matters. 
If the representation is not sufficient to identify $r$, then the RL algorithm cannot distinguish what the current task is and its performance suffers.

Algorithms can also be compared by the landscape of their representations. 
We use cosine similarity to compare representations learned by each algorithm.
An ideal representation should maintain the relationships between functions such that similar functions have similar representations. 
This property would allow the RL algorithm to use similar policies for similar reward functions, whereas if the representations for similar functions are unrelated, the RL algorithm must memorize a separate policy for each reward function.
We graph the cosine similarity of representations in Figure \ref{cos_sim_pacman}.
This graph indicates that the function encoder and the transformer learned a representation that reflects the relationships between reward functions, where similar goals have similar representations. 
In contrast, other approaches do not maintain this relationship.

\section{Conclusion}
We have introduced the function encoder, a general-purpose representation learning algorithm capable of encoding a function as a linear combination of learned, non-linear basis functions. 
The function encoder is a linear operator, meaning the learned representations are generalizable and predictable with respect to previously seen representations. 
Using a function encoder to represent tasks  allows a basic RL algorithm to achieve zero-shot transfer between a set of related tasks.
The representation is simply passed into the policy and value function as an additional input without making major modifications to the RL algorithm.
This method is stable, data efficient, and achieves high asymptotic performance relative to prior approaches while maintaining the simplicity of basic RL algorithms. 

\section{Acknowledgements}

Thank you to my colleagues Cyrus Neary, Sophia Smith, and Dr. Adam Thorpe for helpful discussions. This work was supported in part by NSF 2214939, AFOSR FA9550-19-1-0005, and ARL W911NF-21-1-0009. 

\section{Impact Statement}
This paper presents work whose goal is to advance the field of Machine Learning. There are many potential societal consequences of our work, none which we feel must be specifically highlighted here.

\bibliographystyle{style/icml2024}
\bibliography{sample}

\clearpage
\appendix
\section{Appendix}
\subsection{Proof for Equation 2} \label{proof.eq2}
\[ c_i = \langle f, g_i \rangle = \int_{\mathcal{X}} f(x) g_i(x) dx. \tag{2} \]
This can be shown starting from \eqref{eq:1} (converted to vector notation):
\[ f = \sum_{j=1}^k c_j g_j \tag{1}\]
\[\langle f, g_i \rangle = \langle \sum_{j=1}^k c_j g_j, g_i \rangle\]
\[ \langle f, g_i \rangle = \sum_{j=1}^k c_j  \langle g_j, g_i \rangle\]
\[ \langle f, g_i \rangle = c_i\]
Note the last step is valid due to the orthogonality of the basis functions: 
\[\forall i \neq j \; \langle g_i, g_j \rangle = 0\]
\[ \forall i \; \langle g_i, g_i \rangle = 1\]



\subsection{Implementation Details} \label{implementationdetails}
Code is available here: \newline
{\small \href{https://github.com/tyler-ingebrand/FunctionEncoderRL}{https://github.com/tyler-ingebrand/FunctionEncoderRL}}

\paragraph{Hardware}
All experiments on performed on a 9th generation Intel i9 CPU and a Nvidia Geforce 2060 with 6 GB of memory. 

\paragraph{Transformers}All experiments involving transformers maximize data input size to use all GPU memory.
Furthermore, gradient accumulation is used to improve the input size, which greatly increases training time. 
The maximum input size is $200-400$ examples, depending on the experiment.
Additionally, all transformers are used without positional embeddings. 
In principle, a transformer with a positional embedding can be used as the underlying basis functions for a function encoder, capturing the benefits of both approaches. 

\paragraph{Volume} Equation 4 requires the volume of the input space, $V$, to calculate the coefficients for a given function. However, $V$ may be hard to calculate depending on the input space. For example, the input space may be a unknown subset of $\mathbb{R}^n$ or even an image. In that case, it is unclear how to calculate $V$. To overcome this issue, define the inner product as $\langle f, g_i \rangle = \frac{1}{V}\int_{\mathcal{X}} f(x) g_i(x) dx$. This is still a valid inner product, but $V$ will cancel out in the resulting Monte Carlo integration. This changes the magnitude of the basis functions, since they must either increase or decrease to compensate depending on the value of $V$, but does not require explicit knowledge of $V$ and is thus better in practice.  It also affects the magnitude of the gradients. For this reason, gradient clipping is a useful technique for function encoders. 

\paragraph{Biased Gradients} When training a function encoder, it is important that every gradient update includes gradients from a large number of functions in the function set. If a single function is used to compute loss, the resulting gradients are biased to improve the function encoder's performance with respect to that function but at the cost of decreased performance for other functions. By calculating loss using multiple functions, that bias is reduced. Experimental results show that function encoders (and transformers) trained on one function at a time fail to converge, while using even just five functions at a time will converge. Using more functions per gradient update further improves convergence speed. 

Additionally, each function used to calculate loss ideally should use overlapping data points such that the function can also learn how functions differ for the same input. Without this information, the function encoder may overfit a portion of the input space to a particular function instead of learning how each function fits that space.

\subsection{Inductive Biases}

Since function encoder representations have known properties, they allow the algorithm designer to investigate the inductive biases created by the choice of neural network architecture. These inductive biases can either help or hinder learning, depending on whether or not they align with the problem setting. 

Consider a multi-task environment where only the reward function differs between episodes. There are known useful properties between reward functions and value functions which can be exploited. Suppose the reward function can be written as a linear combination of basis  functions, such as its function encoder representation. Then a linear change in reward leads to a linear change in value for a \textit{given} policy. This implies a good inductive bias for the state-action-value function is
\[Q^\pi(s,a,c_r) := Q^\pi(s,a)^\top c_r, \]
where $Q^\pi(s,a)$ is a vector-valued function where each entry represents the value of the policy with respect to a given reward function basis. 
This inductive bias encapsulates the linear nature of value with respect to a change in reward for a particular policy.
However, the optimal policy is not constant with respect to the reward function, so this inductive bias is poor for the \textit{optimal} policy.

Since a small change in reward function may lead to an abrupt, discontinuous change in the optimal policy, it is necessary that the value function for the optimal policy reflects this. A reasonable architecture is
\[Q^{\pi^*}(s,a,c_r) := Q^{\pi^*}(s,a, c_r)^\top c_r. \]
This inductive bias directly captures the linear relationship of value with respect to reward for the case where the optimal policy does not change with respect to a small change in $c_r$, while also allowing the value function to make abrupt, non-linear changes with respect to $c_r$ if needed. Thus, this architecture has a good inductive bias for state-action-value functions with respect to reward functions because it naturally captures the expected relationship between reward and value. 

\subsection{Piece-Wise Continuous with Sparse Discontinuities} \label{RewTranPol}

The following simple examples illustrate how the relationship between reward or transitions functions and optimal policies may be piece-wise continuous with sparse discontinuities. The environment is a basic grid world where the agent can move left, right, up, or down at some fixed velocity.  

\paragraph{Reward Function}
\begin{figure}           
    \centering
    \resizebox{\linewidth} {!} {%
        \includegraphics{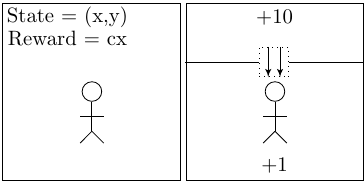}
        }

    \caption{A diagram depicting example MDPs where the reward or transition function varies. This diagram is used to illustrate how policies may vary with respect to changes in their reward (left) or transition functions (right).  
    In the figure on the right, the box with arrows in it indicates a treadmill. If the agent is faster than the treadmill, it can pass over the treadmill. Otherwise, the treadmill would push the agent backwards. }
    \label{envs}
\end{figure}

Consider the environment shown in Figure \ref{envs}. Since the reward is $cx$, if $c>0$, the agent should go right, and if $c<0$, the agent should go left. At $c=0$, there is a discontinuity where the optimal policy changes. For most of the reward function space, a small change in $c$ has no affect on the optimal policy. Hence, the optimal policy is (piece-wise) continuous with respect to a change in reward. However, around $c=0$, a small change in reward leads to a discontinuous change in policy. Thus, there are sparse discontinuities. This example environment illustrates how reward functions can affect optimal policies.

\paragraph{Transition Function} 

Consider the environment shown in Figure \ref{envs}. The treadmill, shown as a rectangle with arrows, has a variable speed $v_{treadmill}$. If the agent's speed $v_{agent}$ exceeds the treadmill's speed, it can move into the upper room and collect $+10$ reward. Otherwise, the treadmill is too strong and pushes the agent back into the room on the bottom, so it can only collect the $+1$ reward. Therefore, the optimal policy of the agent depends on its max speed. If $v_{agent} > v_{treadmill}$, it should go up and collect the $+10$. If $v_{agent} < v_{treadmill}$, it should go down and collect the $+1$. For most treadmill speeds, a small change in speed does not affect the optimal policy. For example, if the agent is much faster than the treadmill, then making the treadmill slightly faster will not affect the optimal policy. However, if the agent is only barely faster than the treadmill, then making the treadmill slightly faster will lead to a discontinuous change in policy. This example illustrates how the optimal policy varies with respect to a change in the transition function. 

\subsection{Orthonormality} \label{app:ortho}
\begin{figure}[h]
    \centering
    \includegraphics[width=\linewidth]{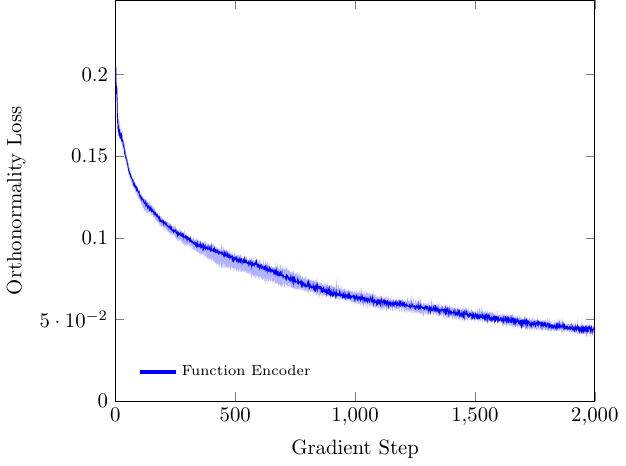}
    \caption{This figure shows the orthonormality of the learned basis functions throughout training on the hidden-parameter system identification task. Y axis indicates a measure of how far the basis functions are from orthonormality. In this example, the orthonormality loss is \textbf{not} used for back-propagation, it is only used to observe how orthonormal the basis functions are. Plot shows min, max, and median over 5 seeds.}
    \label{fig:ablation_o}
\end{figure}

\begin{figure*}[h!]
    \centering
    \begin{subfigure}[t]{0.45\textwidth}
        \includegraphics[width=\linewidth]{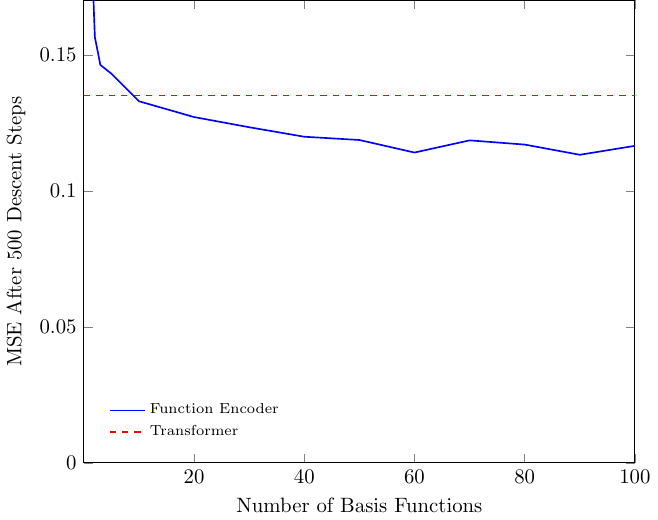}
        \caption{This figure ablates the function encoder's performance on the hidden-parameter system identification task from section 5.1. X-axis indicates the number of basis functions used for training. Y-axis indicates the MSE after 500 descent steps. The red, dashed line indicates the performance of a transformer. }
        \label{fig:ablation_k}
    \end{subfigure}%
    \hfill%
    \begin{subfigure}[t]{0.45\textwidth}
       \includegraphics[width=\linewidth]{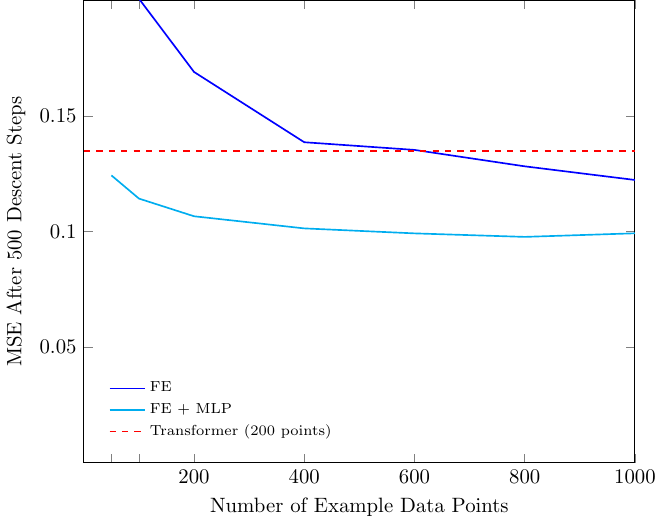}
        \caption{This figure ablates the function encoder's performance on the hidden-parameter system identification task from section 5.1. X-axis indicates the number of example data points. Y-axis indicates the MSE after 500 descent steps. The red, dashed line indicates the performance of a transformer, which is limited to 200 example data points by memory constraints. }
        \label{fig:ablation_n}
    \end{subfigure}

\end{figure*}

We empirically observe that the basis functions converge towards orthonormality during training.  We measure orthonormality via the following term. For each pair of basis functions $g_i,g_j$, the inner product is approximated via Monte Carlo integration. For $i=j$, the inner product would be $1$ if the functions are orthonormal. For $i \neq j$, the inner product would be $0$. The orthonormality loss measures how far the calculated inner products are from these values via mean square error, where the mean is over all pairs of basis functions.
See Figure \ref{fig:ablation_o}. Additionally, we experimented with enforcing orthnormality by using the orthonormality loss as an additional loss component. We found that while it does improve the convergence speed of the orthnormality loss, it does not improve accuracy.

\subsection{Ablations} \label{sec:ablations}
We investigate the effects of the number of basis functions and the number of example data points used to compute the representation. See Figures \ref{fig:ablation_k} and \ref{fig:ablation_n}. 

The ablation shows that the performance of the function encoder is superior to the transformer baseline if more than $10$ basis functions are used. For less than $10$ basis functions, the performance degrades significantly for this particular dataset. Note that the number of basis functions chosen determines the maximum dimensionality of the space that can be learned, and so a larger number of basis functions is better. Furthermore, the number of basis functions needed depends on the dimensionality of the function space in the dataset. We would like to highlight that the function encoder can efficiently use $100$ or more basis functions due to parameter sharing. Therefore, a user may simply choose a large number of basis functions to avoid issues.

We perform an ablation on small data settings, ranging from 50 to 1000 example data points. Results indicate that the function encoder outperforms the transformer if the number of data points is greater than 600. We highlight that the \texttt{FE + MLP} approach is designed for low data settings, and outperforms the transformer even under low data settings.

\end{document}